\documentclass[conference]{IEEEtran}
\IEEEoverridecommandlockouts

\usepackage[table,xcdraw]{xcolor}
\usepackage{cite}
\usepackage{amsmath,amssymb,amsfonts}
\usepackage{algorithmic}
\usepackage{graphicx}
\usepackage{textcomp}
\usepackage{xcolor}
\usepackage{multirow}

\usepackage{bm}
\usepackage{tikz}
\usetikzlibrary{positioning}
\usetikzlibrary{decorations,arrows}
\usetikzlibrary{decorations.markings}
\usetikzlibrary{patterns,calc}
\usetikzlibrary {arrows.meta}
\usetikzlibrary {bending}
\usepackage{algorithm}
\usepackage{graphicx}
\usepackage{makecell}
\usepackage{amsthm}

\newcommand{\indep}{\mathop{\perp\!\!\!\!\perp}}
\newcommand{\notindep}{\mathop{\not \perp\!\!\!\!\perp}}

\makeatletter
\newenvironment{breakablealgorithm}
  {
   \begin{center}
     \refstepcounter{algorithm}
     \hrule height.8pt depth0pt \kern2pt
     \renewcommand{\caption}[2][\relax]{

       {\raggedright\textbf{\ALG@name~\thealgorithm} ##2\par}%
       \ifx\relax##1\relax 
         \addcontentsline{loa}{algorithm}{\protect\numberline{\thealgorithm}##2}%
       \else 
         \addcontentsline{loa}{algorithm}{\protect\numberline{\thealgorithm}##1}%
       \fi
       \kern2pt\hrule\kern2pt
     }
  }{
     \kern2pt\hrule\relax
   \end{center}
  }

\makeatother

\def\BibTeX{{\rm B\kern-.05em{\sc i\kern-.025em b}\kern-.08em
    T\kern-.1667em\lower.7ex\hbox{E}\kern-.125emX}}

\newtheorem{definition}{Definition}
\newtheorem{theorem}{Theorem}

\newtheorem{proposition}{Proposition}
\begin{document}

\title{Learning causal graphs using variable grouping according to ancestral relationship}

\author{\IEEEauthorblockN{Ming Cai}
\IEEEauthorblockA{\textit{Graduate School of Informatics} \\
\textit{Kyoto University}\\
Kyoto, Japan \\
cai.ming.52d@st.kyoto-u.ac.jp}
\and
\IEEEauthorblockN{Hisayuki Hara}
\IEEEauthorblockA{\textit{Institute for Liberal Arts and Sciences} \\
\textit{Kyoto University}\\
Kyoto, Japan \\
hara.hisayuki.8k@kyoto-u.ac.jp}
}

\maketitle

\begin{abstract}

    Causal discovery has drawn increasing attention in recent years. 
    Over the past quarter century, several useful algorithms for learning causal graphs have been proposed. 
    However, when the sample size is small relative to the number of variables, the accuracy of estimating causal graphs using existing methods decreases. 
    In addition, some methods are not feasible when the sample size is smaller than the number of variables.
    
    To circumvent these problems, some researchers proposed causal structure learning algorithms using divide-and-conquer approaches (e.g., \cite{Cai2013, Zhang2020}). For learning the entire causal graph, the divide-and-conquer approaches first split variables into several subsets according to the conditional independence relationships among the variables, then apply a conventional causal structure learning algorithm to each subset and merge the estimated results. 
    Since the divide-and-conquer approach reduces the number of variables to which a causal structure learning algorithm is applied, it is expected to improve the estimation accuracy of causal graphs, especially when the sample size is small relative to the number of variables and the model is sparse.
    However, existing methods are either computationally expensive or do not provide sufficient accuracy when the sample size is small.
    
    This paper proposes a new algorithm for grouping variables 
    according to the ancestral relationships among the variables, assuming that the causal model is LiNGAM \cite{Shimizu2006}, where the causal relationships are linear, and the exogenous variables are mutually independent and distributed as continuous non-Gaussian distributions. 
    We call the proposed algorithm CAG (\textbf{C}ausal \textbf{A}ncestral-relationship-based \textbf{G}rouping).  
    The time complexity of 
    the ancestor finding in
    CAG is shown to be cubic to the number of variables. 
    Extensive computer experiments confirm that the proposed method outperforms the original DirectLiNGAM \cite{Shimizu2011} 
    without grouping variables and other divide-and-conquer approaches
    not only in estimation accuracy but also in computation time 
    when the sample size is small relative to the number of variables and the model is sparse. 
\end{abstract}

\begin{IEEEkeywords}
causal discovery, causal DAG, conditional independence test, DirectLiNGAM, divide-and-conquer, linear regression.
\end{IEEEkeywords}
\section{Introduction}
\label{sec:1}
In recent years, inference using high-dimensional causal models for observational data has played a pivotal role in various fields, such as econometrics \cite{Hoover2006}, biology \cite{Sachs2005}, and psychology \cite{Glymour1998}. 
For the high dimensional causal inference, the structural causal model (SCM, \cite{Pearl, Pearl1995}) defined by a directed acyclic graph (DAG) has been extensively used. 
Since the causal relationships among variables are usually unknown, we also need to learn the causal DAG before we use the SCM.

Over the past quarter century, several practical algorithms for learning causal DAGs without latent confounders have been proposed. These algorithms are classified into several types. 

The constraint-based methods use the conditional independence (CI) relationships among variables to determine causal directions. The PC algorithm \cite{Spirtes} is a typical example. 
Using hypothesis tests, the PC algorithm first estimates CI relationships between all pairs of variables. Then, it estimates the skeleton of a causal DAG and the directions of edges in the skeleton in turn.
However, the time complexity of the PC algorithm, in the worst case, is exponential to the number of variables. 
Hence, the PC algorithm is not feasible for high-dimensional data. 

The greedy equivalent search (GES) \cite{Chickering2002} is an algorithm that learns causal structures using a model selection criterion such as BIC \cite{schwarz1978}. This type of method is called the score-based method. As Chickering et al. \cite{Chickering1996, Chickering2004} have proven, however, GES belongs to NP-hard, 
and its application to high-dimensional causal models is also impractical.

It is important to note that the constraint-based and score-based methods can only identify causal DAGs up to the Markov equivalence class, the set of all DAGs compatible with the inferred CI relationships.
To fully identify causal DAGs requires additional constraints on the causal model.

Shimizu et al. \cite{Shimizu2006} considered the linear structural equation model as the causal model, where exogenous variables are mutually independent, and the distributions of the exogenous variables are continuous and non-Gaussian. 
They called the model linear non-Gaussian acyclic model (LiNGAM) and showed that it is possible to fully identify the causal DAG that defines LiNGAM using independent component analysis (ICA-LiNGAM) \cite{Hyvarinen2002}. 
However, ICA-LiNGAM tends to converge to a locally optimal solution when the dimension of the model is high, resulting in lower estimation accuracy (e.g.,\cite{Shimizu2011}). To overcome this difficulty, Shimizu \cite{Shimizu2011} proposed DirectLiNGAM that estimates causal DAGs using linear regression between variables and CI tests.

Hoyer et al. \cite{Hoyer2008} and Zhang and Hyv\"arinen \cite{Zhang2012} generalized LiNGAM to nonlinear and showed that when the model is nonlinear, the causal DAG is identifiable even if the exogenous variables are Gaussian distributed.

However, the estimation accuracy of existing methods tends to decrease as the sample size decreases relative to the number of variables. 
Furthermore, DirectLiNGAM is not feasible when the sample size is smaller than or equal to the number of variables (e.g., \cite{Wang2020}).

Cai et al. \cite{Cai2013} proposed the scalable causation 
discovery algorithm (SADA) to address this problem. 
The main idea of SADA is to group variables into several subsets according to the CI relationships between the variables, then learn the causal DAG of each group, merge each result, and return the entire causal DAG estimate.  
In this method, the causal DAG learning algorithm is applied to each smaller group of variables, reducing the sample size required for it to work.
Through simulation studies, Cai et al. \cite{Cai2013} showed that SADA improves LiNGAM concerning estimation accuracy. 
In SADA, however, the time complexity of grouping variables is of exponential order for the number of variables.
In addition, SADA has a severe drawback in that even if the correct CI relationships are known, the marginal model for the variables in each group is not necessarily the causal model defined by the sub-DAG induced by the variables in each group. 
This means the estimator of a causal DAG estimated by SADA may not be consistent depending on the true causal DAG (e.g., \cite{Zhang2020}). 

Zhang et al. \cite{Zhang2020} proposed another algorithm for grouping variables called causality partitioning (CAPA). 
CAPA requires only low-order CI relationships among variables to group variables. 
Letting $p$ be the number of variables and $\sigma_{\max}$ be the maximum order of the conditioning set, the time complexity of CAPA is $O(p^{\sigma_{\max}+2})$. 
CAPA guarantees that if the correct low-order CI relationships are known, the marginal model for each group of variables is the causal model defined by the sub-DAG induced by the variables in each group.
However, especially in the case of DAGs with high outdegree and low indegree, CAPA tends not to group variables finely enough to improve estimation accuracy even if the true conditional independence relationships are known.

Maeda and Shimizu \cite{Maeda2020} proposed the repetitive causal discovery (RCD), which is intended to be applied to the model with latent confounders.
RCD can also be applied to the models without latent confounders.
RCD first determines the ancestral relationships among variables using linear regression and CI tests and creates a list of ancestor sets for each variable. 
The parent-child relationships between variables are determined from the CI relationships among variables in each estimated ancestor set. RCD can be interpreted as one of the divide-and-conquer approaches because RCD splits the entire variable into families of ancestor sets for each variable and then learns causal DAGs for each group. 

In this paper, we propose another scheme for grouping variables according to the ancestral relationships between the variables. 
We also assume LiNGAM for the causal model. 
The proposed method uses the algorithm to find variables' ancestor sets in RCD. 
When the true causal DAG is connected, and the ancestor sets are correctly estimated, the variable grouping obtained by the proposed method is the family of maximal elements in the family of the union of each variable and its ancestors. 
The time complexity of the proposed algorithm for grouping variables is the third order of the number of variables. 
Suppose the true ancestor sets are correctly estimated. In that case, the marginal model for the variables in each group obtained by the proposed method is always the LiNGAM defined by the sub-DAG induced by the variables in each group, and no edge exists across different groups. 
Therefore, the entire causal DAG can be consistently estimated by applying the causal structure learning algorithm such as DirectLiNGAM to each group and merging the results. 
We call the proposed procedure for grouping variables the causal ancestral-relationship-based grouping (CAG). 

Extensive computer experiments show that CAG outperforms the original DirectLiNGAM 
without grouping variables, CAPA, and RCD regarding estimation accuracy and computation time when the sample size is small relative to the number of variables and the true causal DAG is sparse. 

The rest of this paper is organized as follows: 
Section \ref{sec:2} summarizes some existing causal structure learning algorithms and clarifies the position of the proposed method. 
Section \ref{sec:3} describes the details of the proposed method. 
Section \ref{sec:4} confirms the proposed methods's usefulness through computer experiments.
Section \ref{sec:5} 
concludes the paper. 
The pseudo-code for the proposed method is provided in the Appendix.

\section{Related Works}
\label{sec:2}
\subsection{LiNGAM and Variants}
\label{sec:LiNGAM}
Let $X=(x_1,\ldots,x_p)^\top$ be a $p$-dimensional random vector.
In the following, we identify $X$ with the variable set. 
Assume that the causal relationship among the variables is acyclic and linear. 
LiNGAM \cite{Shimizu2006} is expressed by 
\begin{align}
    \label{model:LiNGAM}
        X  =  BX + e,
\end{align}
where the error term $e=(e_1,\ldots,e_p)^\top$ is assumed to be independently distributed as continuous non-Gaussian distributions.
$B=\{b_{ij}\}$ is a $p\times p$ coefficient matrix. 
$b_{ij}$ represents the direct causal effect from $x_j$ to $x_i$.
$b_{ij}=0$ indicates the absence of the direct causal effect from $x_j$ to $x_i$.
We note that it is possible to transform the matrix $B$ into a strictly lower triangular matrix by permuting the rows and columns when a DAG defines the causal relationship between $X$. 

The model (\ref{model:LiNGAM}) is rewritten by 
\begin{align}
    \label{model:LiNGAM2}
    X = (I-B)^{-1}e,
\end{align}
where $I$ denotes the $p \times p$ identity matrix. 
Noting that (2) is equivalent to the independent component analysis (ICA) model, 
Shimizu et al. \cite{Shimizu2006} showed that $B$ is identifiable and proposed an algorithm for estimating $B$.
The algorithm is known as ICA-LiNGAM. 

However, as the number of variables increases, ICA-LiNGAM is more likely to converge to a locally optimal solution, which reduces the accuracy of causal DAG estimation. 
To overcome this problem, Shimizu et al. \cite{Shimizu2011} proposed another algorithm for estimating $B$ named DirectLiNGAM.  
When linear regression analyses are conducted in two ways for each pair of variables, with one of them as the dependent variable and the other as the independent variable, if the independent variable and the residuals are mutually independent in only one of the models, the independent variable in that model will precede the dependent variable in causal order. 
After determining the causal order of all variables, DirectLiNGAM estimates the set of parents of each variable by sparse estimating a linear regression model with each variable as the dependent variable and all variables preceding it in the causal order as independent variables. This approach ensures no edges violate the estimated causal order. 
DirectLiNGAM is also less accurate when the number of variables is large relative to the sample size.
DirectLiNGAM tends to output redundant edges even when the sample size is large.

The computational cost of DirectLiNGAM is higher than that of ICA-LiNGAM. 
The time complexity of DirectLiNGAM is 
$O(np^3M^2 + p^4M^3)$, where $n$ is the sample size, 
$p$ is the number of variables, 
and $M$ is the maximal rank found by the low-rank decomposition used in the kernel-based independence measure \cite{Shimizu2011}. 
Since DirectLiNGAM requires estimating a linear regression model, 
DirectLiNGAM is not feasible when $n$ is smaller than $p$. 
\subsection{Scalable Causation Discovery Algorithm (SADA)}
\label{sec:SADA}
Cai et al. \cite{Cai2013} proposed a divide-and-conquer approach called the scalable causation discovery algorithm (SADA) to enhance the scalability of causal DAG learning algorithms. 
SADA first splits the variable set $X$ into two or more subsets. 
Then, a causal DAG learning algorithm such as DirectLiNGAM is applied to each group of variables, and finally, 
all the results are merged to estimate the entire causal DAG.
Because SADA applies the causal DAG learning algorithm to each group with a smaller number of variables, if the variables are grouped correctly, SADA is expected to improve the estimation accuracy when the sample size is small relative to the number of the entire variable set.
SADA is feasible even if the sample size is smaller than the number of variables as long as it is larger than the number of variables in each group. 

The procedure for grouping $X$ into two subsets by SADA is as follows.
SADA first randomly selects two variables $x_i, x_j \in X$ satisfying $x_i \indep x_j \mid (X \setminus \{x_i,x_j\})$ 
and finds the smallest $\hat{X} \subseteq X\backslash \lbrace x_i, x_j \rbrace$ with respect to the inclusion relation 
satisfying $x_i \indep x_j \mid \hat{X}$. 
The disjoint subsets $X_1$, $X_2$, and $C$ are computed according to the following procedure.

\begin{enumerate}
    \item Initialize with $X_1=\{x_i\}$, $X_2=\{x_j\}$, $C = \{\hat{X}\}$.
    \item For all $w \in X \setminus (X_1 \cup X_2 \cup C)$
        \begin{itemize}
            \item[(a)] if $w \indep X_2 \mid \hat{C}$ for $\exists\hat{C} \subseteq C$, \\
                then $X_1 \leftarrow X_1 \cup \{w\}$
            \item[(b)] if $w \indep X_1 \mid \hat{C}$ for $\exists\hat{C} \subseteq C$,\\
                then $X_2 \leftarrow X_2 \cup \{w\}$
            \item[(c)] else $C \leftarrow C \cup \{w\}$
        \end{itemize}
    \item For all $s \in C$
    \begin{itemize}
            \item[(a)] if $s\indep X_2 \mid \hat{C}$ for $\exists\hat{C} \subseteq C \setminus \{ s\}$, \\
                then $X_1 \leftarrow X_1 \cup \{s\}$ and $C \leftarrow C \setminus \{s\}$
            \item[(b)] if $s\indep X_1 \mid \hat{C}$ for $\exists\hat{C} \subseteq C \setminus \{ s\}$,\\
                then $X_2 \leftarrow X_2 \cup \{s\}$ and $C \leftarrow C \setminus \{s\}$
        \end{itemize}
    \item Return $X_1$, $X_2$ and $C$
\end{enumerate}
In step 2, the intermediate set $C$ tends to be large, and step 3 downsizes $C$ as much as possible.

The output of SADA is two subsets $V_1 = X_1 \cup C$ and $V_2 = X_2 \cup C$. 
This algorithm can be repeated recursively to group variables into smaller subsets. 
In the implementation, the lower bound of the number of variables in each group is set to $\theta$, and SADA is applied to $V_1$ or $V_2$ to create even smaller groups only when $|V_1|>\theta$ or $|V_2|>\theta$. If the number of variables is less than or equal to $\theta$, or if we cannot find $x_i, x_j \in V_k$ satisfying $x_i \indep x_j \mid (V_k \setminus \{x_i,x_j\})$, $k=1,2$, 
no further grouping is made.

We illustrate the SADA procedure to split $X$ into two subsets where the true causal DAG is the one in Fig \ref{Example_Fig}.
Table \ref{SADA1} shows the process of grouping variables $x_1,\ldots,x_9$ into two subsets. 
Since both $x_3 \indep x_9 \mid \{x_1, x_2, x_4\ldots x_8\}$ and $x_3 \indep x_9$ hold in the model defined by the DAG in Fig. \ref{Example_Fig}, 
we can initialize with $X_1=\{x_3\}$, $X_2=\{x_9\}$ and $C=\emptyset$. 
We then use the CI tests to determine one by one to which of $X_1$, $X_2$, and $C$ the remaining variables belong.
This procedure returns $X_1=\{x_1,x_2,x_3, x_5,x_6\}$, $X_2=\{x_7,x_8,x_9\}$, $C=\{x_4\}$ and hence
$V_1=\lbrace x_1,x_2,x_3,x_4,x_5,x_6 \rbrace$ and $V_2 = \lbrace x_4, x_7,x_8,x_9\rbrace$. 
In this case, we note that the marginal models for $V_1$ and $V_2$ are defined by the sub-DAGs induced by $V_1$ and $V_2$, respectively. 
If the initial values of $X_1$, $X_2$, and $C$ or the scanning order of $w \in X \setminus (X_1 \cup X_2 \cup C)$ changes, $V_1$ and $V_2$ may also change. 

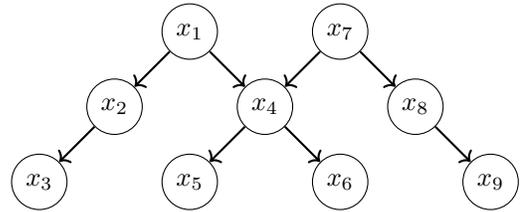
\begin{figure}[!ht]
    \centering
    \scalebox{1}{
        \begin{tikzpicture}            
            \node[draw, circle] (A) at (-1,1) {$x_1$};
            \node[draw, circle] (B) at (-2,0) {$x_2$};
            \node[draw, circle] (C) at (-3,-1) {$x_3$};
            \node[draw, circle] (D) at (0,0) {$x_4$};
            \node[draw, circle] (E) at (-1,-1) {$x_5$};
            \node[draw, circle] (F) at (1,-1) {$x_6$};
            \node[draw, circle] (G) at (1,1) {$x_7$};
            \node[draw, circle] (H) at (2,0) {$x_8$};
            \node[draw, circle] (I) at (3,-1) {$x_9$};
            \draw[->, thick] (A) -- (B);
            \draw[->, thick] (B) -- (C);
            \draw[->, thick] (A) -- (D);
            \draw[->, thick] (D) -- (E);
            \draw[->, thick] (D) -- (F);
            \draw[->, thick] (G) -- (D);
            \draw[->, thick] (G) -- (H);
            \draw[->, thick] (H) -- (I);
        \end{tikzpicture}
    }
    \caption{A causal DAG with nine variables.}
    \label{Example_Fig}
\end{figure}

\begin{table}[!ht]
\centering
\caption{The grouping process of the variables in Fig. \ref{Example_Fig} in SADA.}
\begin{tabular}{c|c|c|c|c}
\hline
 Step & $X$ & $X_1$ & C & $X_2$ \\ \hline
 Initial & \makecell{$x_1, x_2$, $x_4$, \\$x_5$, $x_6$, $x_7$, $x_8$} & $x_3$ & $\phi$ & $x_9$ \\ \hline
 Check $x_1$ & \makecell{$x_2$, $x_4$, $x_5$, \\$x_6$, $x_7$, $x_8$} & $x_1$, $x_3$ & $\phi$ & $x_9$ \\ \hline
Check $x_2$ & \makecell{$x_4$, $x_5$, $x_6$, \\$x_7$, $x_8$} & $x_1$, $x_2$, $x_3$ & $\phi$ & $x_9$ \\ \hline
Check $x_4$ & \makecell{$x_5$, $x_6$, $x_7$, \\$x_8$} & $x_1$, $x_2$, $x_3$ & $x_4$ & $x_9$ \\ \hline
Check $x_5$ & $x_6$, $x_7$, $x_8$ & \makecell{$x_1$, $x_2$, $x_3$ \\ $x_5$} & $x_4$ & $x_9$ \\ \hline
Check $x_6$ & $x_7$, $x_8$ & \makecell{$x_1$, $x_2$, $x_3$ \\ $x_5$, $x_6$} & $x_4$ & $x_9$  \\ \hline
Check $x_7$ & $x_8$ & \makecell{$x_1$, $x_2$, $x_3$ \\ $x_5$, $x_6$} & $x_4$ & $x_7$, $x_9$ \\ \hline
Check $x_8$ & $\phi$ & \makecell{$x_1$, $x_2$, $x_3$ \\ $x_5$, $x_6$} & $x_4$ & $x_7$, $x_8$, $x_9$ \\ \hline
\end{tabular}
\label{SADA1}
\end{table}

Suppose that $X$ is grouped into $V_1,\ldots, V_K$ by the SADA's variable grouping procedure. 
SADA applies a causal structure learning algorithm such as DirectLiNGAM to each of the $K$ groups.
Let $\hat{G}_1=(V_1, E_1),\ldots,\hat{G}_K=(V_K, E_K)$ be the estimated $K$ DAGs for each group, where $E_k \subset V_k \times V_k$, $k=1,\ldots,K$ are the set of directed edges in $\hat{G}_k$. 
Then, the entire causal DAG is estimated by $\hat{G}=(X, E_1\cup \cdots \cup E_K)$.

However, SADA has several practical problems.
In SADA, the marginal model $\hat{G}_k$, $k=1,\ldots, K$ is not necessarily the causal model 
induced by $V_k$ even if we knew the correct CI relationships between variables.  
Assume that the true causal DAG is Fig \ref{Example_Diamond} and that $\theta \le 3$. 
Since $x_1 \indep x_4 \mid (x_2,x_3)$, 
we can initialize with $X_1 = \lbrace x_1 \rbrace$, $X_2 = \lbrace x_4 \rbrace$, and $C = \lbrace x_2, x_3 \rbrace$, which does not need further variable checking. 
As a result, we group the original variable set into $V_1 = \lbrace x_1, x_2, x_3\rbrace$ and $V_2 = \lbrace x_2, x_3, x_4\rbrace$. 
For $V_1$, the marginal model is defined by 
the sub-DAG induced by 
$x_1$, $x_2$ and $x_3$.
However, 
the marginal model with respect to $V_2$ is not the model defined by the sub-DAG induced by $(x_2,x_3,x_4)$ because $x_2 \notindep x_3$. 

\begin{figure}[!hbt]
    \centering
    \scalebox{1}{
        \begin{tikzpicture}            
            \node[draw, circle] (A) at (0,1) {$x_1$};
            \node[draw, circle] (B) at (-1,0) {$x_2$};
            \node[draw, circle] (C) at (1,0) {$x_3$};
            \node[draw, circle] (D) at (0,-1) {$x_4$};
            
            \draw[->, thick] (A) -- (B);
            \draw[->, thick] (A) -- (C);
            \draw[->, thick] (B) -- (D);
            \draw[->, thick] (C) -- (D);
        \end{tikzpicture}
    }
    \caption{An example of SADA not working.}
    \label{Example_Diamond}
\end{figure}
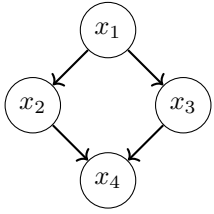

SADA does not guarantee that it always returns a DAG. 
Cycles may appear when DAG estimates for each variable group are merged. 
Suppose $V_1$ and $V_2$ share $x_1$ and $x_2$. If $x_1 \to x_2 \in E_1$, $x_2 \to x_1 \in E_2$, the cycle $x_1 \to x_2 \to x_1$ appears after merging $\hat{G}_1$ and $\hat{G}_2$.
Therefore, SADA needs to remove redundant edges in cycles. 

In addition, SADA requires too many CI tests for grouping variables. 
Consider the simple model where $x_1,\ldots,x_p$ are mutually independent and no edge exists in the true DAG. 
In this case, any $x_i$ and $x_j$ satisfy $x_i \indep x_j \mid (X \setminus \{x_i,x_j\})$.
In the worst case, $2^{p-2}$ CI tests are required to ensure $\hat{C}=\emptyset$. 
Thus, SADA is not feasible for high-dimensional data.

\subsection{Causality Partitioning (CAPA)}
\label{sec:CAPA}
Zhang et al. \cite{Zhang2020} proposed a refined grouping scheme called the causality partitioning (CAPA). 
Define the order of the CI test by the number of variables in the conditioning set.
CAPA reduces time complexity by setting an upper limit on the order of CI tests and using only CI tests of orders below that limit.

When the maximum order of the CI test is $\sigma_{\max}$, the time complexity of CAPA is shown to be 
$O(p^{\sigma_{\max} + 2})$. 
In implementation, $\sigma_{\max}$ is set to small numbers, like two or three.

The procedure for grouping $X$ into two subsets by CAPA is as follows.
Let $\sigma$ be the current order of the conditioned set. 
Let $M^\sigma = \{M^\sigma_{ij}\}$ be a $p \times p$ matrix such that
\begin{equation*}
    M^\sigma_{ij} = \left\{
        \begin{array}{ll}
            0&  x_i \indep x_j \mid Z, \exists Z \subset X \setminus \{x_i,x_j\} \text{ s.t. } |Z|=\sigma\\
            1 &  \text{otherwise}.
        \end{array}
        \right.
\end{equation*}
Let $G^\sigma = (X, E^\sigma)$ be an undirected graph with $M^\sigma$ as its adjacency matrix. 
Let $d_i = \sum_{r=1,r \ne i}^p M^\sigma_{ir}$, $i=1,\ldots,p$ be the degrees of $x_i$ in $G^\sigma$. 

\begin{enumerate}
    \item Initialize with $\sigma = 0$
    \item Initialize with $A=B=C=D=\emptyset$
    \item For all $x_i \in X$ in ascending order by $d_i$, $i=1,\ldots,p$
        do at most one of the followings so that $A \ne \emptyset$, $B \ne \emptyset$ at output
        \begin{itemize}
            \item[(a)] $A \leftarrow A \cup \{x_i\}$ if $\forall x_j \in B$, $(x_i,x_j)\notin E_\sigma$ 
            \item[(b)] $B \leftarrow B \cup \{x_i\}$ if $\forall x_j \in A$, $(x_i,x_j)\notin E_\sigma$
        \end{itemize}
    \item $C \leftarrow X \setminus (A \cup B)$
    \item For all $x_i \in A \cup B$\\
        $D \leftarrow D \cup \{x_i\}$ if $\exists x_j \in C$, 
        $(x_i,x_j) \in E_\sigma$
    \item $X_1 = A \cup C \cup D$, $X_2 = B \cup C \cup D$
    \item Return $\{X_1,X_2\}$ if $\max(|X_1|,|X_2|) < |X|$\\
        Else if $\sigma=\sigma_{\max}$, exit\\
        Else $\sigma \leftarrow \sigma+1$ and go to 2
\end{enumerate}

Let $S=X_1 \cap X_2$. 
When we know the true CI relationship between $X$, 
$S$ always d-separates $X_1 \setminus S$ and $X_2 \setminus S$ in the true causal DAG. 
Therefore, the marginal model for $X_1$ and $X_2$ is guaranteed to be defined by the sub-DAG induced by $X_1$ and $X_2$, respectively. 
In general, the ascending order of $d_i$, $i=1,\ldots,n$ is not unique, and the output $\{X_1, X_2\}$ may change depending on the choice of the ascending order. 
In Step 3, both 3(a) and 3(b) may be possible depending on $x_i$, in which case one of them is randomly selected, which may also change the output. 

We illustrate the CAPA procedure with $\sigma=1$ to create two subsets where the true causal DAG is the one in Fig 
\ref{Example_Fig}. 
When $\sigma=1$, $G_\sigma$ is an undirected graph in Fig. \ref{Example_Fig2}. 
Table \ref{CAPA1} shows an example of the results of Step 3 and Step 4 when $x_3,x_5,x_6,x_9,x_2,x_8,x_1,x_7,x_4$ is used as the order of $x_i$ by $d_i$. 
Since $C=\{x_4\}$, we have $D=\{x_1,x_5,x_6, x_7\}$.
Therefore $X_1$ and $X_2$ are 
\begin{align*}
    X_1&=\{x_1,x_2,x_3,x_4,x_5,x_6,x_7\}\\
    X_2&=\{x_1,x_4,x_5,x_6,x_7,x_8,x_9\}, 
\end{align*}
respectively. 

In this example, $|X_1|=7$, $|X_2|=7$, so the size of each group is not so small. 
CAPA can also be repeated recursively to split variables into smaller subsets. 
CAPA tends to fail to group finely in causal DAGs that contain vertices with high outdegree and low indegree.

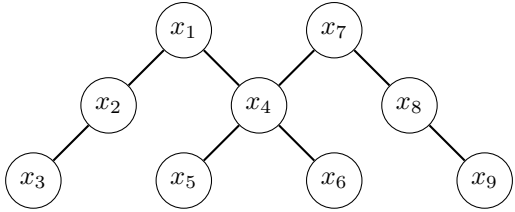
\begin{figure}[!ht]
    \centering
    \scalebox{1}{
        \begin{tikzpicture}            
            \node[draw, circle] (A) at (-1,1) {$x_1$};
            \node[draw, circle] (B) at (-2,0) {$x_2$};
            \node[draw, circle] (C) at (-3,-1) {$x_3$};
            \node[draw, circle] (D) at (0,0) {$x_4$};
            \node[draw, circle] (E) at (-1,-1) {$x_5$};
            \node[draw, circle] (F) at (1,-1) {$x_6$};
            \node[draw, circle] (G) at (1,1) {$x_7$};
            \node[draw, circle] (H) at (2,0) {$x_8$};
            \node[draw, circle] (I) at (3,-1) {$x_9$};
            \draw[thick] (A) -- (B);
            \draw[thick] (B) -- (C);
            \draw[thick] (A) -- (D);
            \draw[thick] (D) -- (E);
            \draw[thick] (D) -- (F);
            \draw[thick] (G) -- (D);
            \draw[thick] (G) -- (H);
            \draw[thick] (H) -- (I);
        \end{tikzpicture}
    }
    \caption{$G_\sigma$ with $\sigma=1$ for the DAG in Fig.\ref{Example_Fig}}
    \label{Example_Fig2}
\end{figure}

\begin{table}[!ht]
\centering
\caption{The grouping process of the variables in Fig. \ref{Example_Fig}.}
\begin{tabular}{c|c|c|c|c}
\hline
 order & $d_i$ & $A$ & $B$ & $C$\\ \hline
 $x_3$ & 1 & $x_3$ & $\emptyset$ & $\emptyset$\\ \hline
 $x_5$ & 1 & $x_3$ & $x_5$ & $\emptyset$\\ \hline
 $x_6$ & 1 & $x_3$ & $x_5, x_6$ & $\emptyset$\\ \hline
 $x_9$ & 1 & $x_3$ & $x_6, x_6,x_9$ & $\emptyset$\\ \hline
 $x_2$ & 2 & $x_2,x_3$ & $x_5, x_6,x_9$ & $\emptyset$ \\ \hline
 $x_8$ & 2 & $x_2,x_3$ & $x_5,x_6,x_8,x_9$ & $\emptyset$  \\ \hline
 $x_1$ & 2 & $x_1,x_2,x_3$ & $x_5,x_6,x_8,x_9$ & $\emptyset$ \\ \hline
 $x_7$ & 2 & $x_1,x_2,x_3$ & $x_5,x_6,x_7,x_8,x_9$ & $\emptyset$\\ \hline
 $x_4$ & 4 & $x_1,x_2,x_3$ & $x_5,x_6,x_7,x_8,x_9$ & $x_4$\\ 
 \hline
\end{tabular}
\label{CAPA1}
\end{table}

\subsection{Repetitive Causal Discovery (RCD)}
\label{sec:RCD}
RCD \cite{Maeda2020} is a novel causal structure learning algorithm that can be applied even when the model contains latent confounders.
RCD also assumes LiNGAM, i.e., the causal relationships are linear, and the exogenous variables are independently distributed as continuous non-Gaussian distributions.
The processes in RCD are divided into the following three parts: ancestral relationship finding, parental relationship finding, and confounder determining. 

The first step in the RCD is to determine the ancestral relationship between variables by repeatedly conducting simple linear regressions and 
independence tests on each variable pair, in a similar way as when determining causal order in DirectLiNGAM, and to create a list of ancestor sets for each variable.
Let $\widehat{Anc}_i$ be the estimated ancestor set of $x_i$.

Once ancestral relationships are estimated, RCD extracts parent-child relationships between each pair of variables using CI tests. 
Assume $x_j \in \widehat{Anc}_i$. 
For $x_i$ and $x_j$, if $x_i  \notindep x_j \mid \widehat{Anc}_i \setminus \{x_j\}$, then $x_j$ can be determined as a parent of $x_i$.
If a variable pair in the RCD's output for which the direction of causality cannot be identified exists, we conclude that there are latent confounders between them.

Our study assumes that the causal model does not contain latent confounders. 
RCD can also be applied to the model without latent confounders. 
RCD is interpreted as another divide-and-conquer approach to improve scalability because it estimates the causal DAG by estimating the parent-child relationship between variables within each variable's ancestor set.
Our proposed method also uses the algorithm to find ancestral relationships in RCD to group variables.

RCD may return a graph that contains cycles even when the sample size is large.
This is because errors in estimating ancestral relationships lead to an incorrect estimation of the parent-child relationships. 

\subsection{Positioning of the Proposed Method}
The proposed method also assumes LiNGAM (\ref{model:LiNGAM}). 
Like SADA and CAPA, the proposed method begins by grouping variables into multiple subsets based on the ancestral relationships between variables. 
The ancestral relationships are estimated in the same way as in RCD. 
When the true causal DAG is connected, the proposed method defines the variable grouping by the maximal elements in the family of the union of each variable and its ancestors. Thus, if the ancestral relationships are correctly estimated, variables are partitioned into as many groups as the number of sink nodes in the true causal DAG. 
The proposed method groups finely for DAGs with high outdegree and low indegree. 
The proposed method's time complexity for variable grouping is $O(p^3)$, which is the same order as CAPA with $\sigma_{\max}=1$.

When the ancestral relationships among the variables are correctly estimated, the marginal model for variables in each group obtained by the proposed method is the LiNGAM defined by the sub-DAG induced by the variables in each group. 

In RCD, the estimated ancestral relationships are used to estimate parent-child relationships among variables. In contrast, the estimated ancestral relationships are only used to group variables in the proposed method. For each group, DirectLiNGAM is applied to estimate sub-DAGs, which are then merged to estimate the entire causal DAG.

Section \ref{sec:3} provides a more comprehensive description of our proposed method.
Section \ref{sec:4} compares the performance of the proposed method, the original DirectLiNGAM, CAPA, and RCD in the absence of latent confounders by computer experiments. 

\section{Causal Ancestral-Relationships-based Groupoing (CAG)}
\label{sec:3}
This section will introduce our proposed algorithm in detail. 
Let $X = (x_1,\ldots,x_p)^\top$ be a $p$-dimensional random vector. 
Let $G=(X, E)$ be the true causal DAG, where $E \subset X \times X$ is the set of edges in $G$. 
We assume that $X$ follows LiNGAM (\ref{model:LiNGAM}) without latent confounders.
Let $Anc_i$ and $Pa_i$ be the sets of ancestors and parents of $x_i$ in $G$, respectively. 
$CA_{ij} := Anc_i \cap Anc_j$ is the set of common ancestors of $x_i$ and $x_j$ in $G$.


\subsection{Finding the Ancestral Relationships Between the Variables}
\label{sec:ancestor}
This subsection summarizes the algorithm for finding the ancestral relationships between variables in $X$ in RCD \cite{Maeda2020}. 
To determine the ancestral relationship between $x_i$ and $x_j$, $i \ne j$, Maeda and Shimizu \cite{Maeda2020} considered the simple regression models
\begin{equation}
\begin{aligned}
    \label{model:no ancestors}
    x_i &= \beta_{ij}x_j + u_i,\\
    x_j &= \beta_{ji}x_i + u_j, 
\end{aligned}
\end{equation}
where $u_i$ and $u_j$ are error terms. 
Maeda and Shimizu \cite{Maeda2020} focused on the independence relationship between the independent variable and the error term in each model.
The following proposition holds for the ancestral relationship between $x_i$ and $x_j$.
\begin{proposition}[Maeda and Shimizu \cite{Maeda2020}]
    \label{Prop:1}
    One of the following four conditions holds for the ancestral relationship between $x_i$ and $x_j$.
    \begin{enumerate}
        \item If $x_i \indep x_j$, then $x_i \notin Anc_j \wedge x_j \notin Anc_i$.
        \item If $x_j \indep u_i$, then $x_i \in Anc_j$.
        \item If $x_i \indep u_j$, then $x_j \in Anc_i$.
        \item If $x_i \notindep u_j \wedge x_j \notindep u_i$, then $CA_{ij} \neq \emptyset$.
    \end{enumerate}
    
\end{proposition}

This algorithm first checks which conditions 1 through 4 in Proposition \ref{Prop:1} holds for each variable pair $(x_i,x_j)$.
For implementation, the error terms $u_i$ and $u_j$ are replaced with the OLS residuals. 
If $(x_i,x_j)$ satisfies condition 4, the determination of the ancestral relationship between $(x_i,x_j)$ is withheld. 
Let $R$ be the set of variable pairs $(x_i, x_j)$ that satisfies condition 4. 
Let $CA^*_{ij}$ be the set of common ancestors of $(x_i,x_j)$ found during checking Proposition 1 
for all pairs $(x_i,x_j)$. 

Assume that $(x_i,x_j) \in R$.  
Then, $CA_{ij}^*$ forms part of the backdoor path for $x_i$ and $x_j$.
To remove the influence of $CA^*_{ij}$ from $x_i$ and $x_j$, consider the regression models 
\begin{equation}
    \begin{aligned}
        \label{model: no ancestors}
        x_i &= \bm{\alpha}_{ij}^\top \cdot CA^*_{ij} + v_i, \\
        x_j &= \bm{\alpha}_{ji}^\top \cdot CA^*_{ij} + v_j,   
    \end{aligned}
\end{equation}
where we identify $CA^*_{ij}$ with the vector of common ancestors of $(x_i,x_j)$. 
Furthermore, consider the following regression models for the error terms $v_i$ and $v_j$, 
\begin{equation}
    \begin{aligned}
        \label{eq: CA1}
        v_i &=  \beta_{ij} v_j + u_i,\\
        v_j &=  \beta_{ji} v_i + u_j.
    \end{aligned}
\end{equation}
Then Proposition \ref{Prop:1} is generalized as follows. 
\begin{proposition}[Maeda and Shimizu \cite{Maeda2020}]
\label{Prop:2}
     One of the following four conditions holds for the ancestral relationship between $(x_i, x_j) \in R$.
    \begin{enumerate}
        \item If $v_i \indep v_j$, then $x_i \notin Anc_j \wedge x_j \notin Anc_i$.
        \item If $v_j \indep u_i$, then $x_i \in Anc_j$.
        \item If $v_i \indep u_j$, then $x_j \in Anc_i$.
        \item If $v_i \notindep u_j$ and $v_j \notindep u_i$, 
        then $CA_{ij} \setminus CA^*_{ij} \ne \emptyset$. 
    \end{enumerate}
\end{proposition}

Proposition \ref{Prop:1} is the case where $CA_{ij}=\emptyset$. 
For implementation, the error terms $v$ and $u$ are replaced with OLS residuals. 
If $(x_i,x_j)$ satisfies condition 4, the determination of the ancestral relationship between $x_i$ and $x_j$ is withheld. After checking Proposition \ref{Prop:2} for all $(x_i, x_j) \in R$, update $CA^*_{ij}$ and $R$. If $R \ne \emptyset$, recheck Proposition \ref{Prop:2}. 
Theoretically, if the model contains no latent confounders, the procedure in Proposition \ref{Prop:2} can be repeated until $R = \emptyset$ to completely determine the ancestral relationship of all $(x_i,x_j)$.

The pseudo-code of this algorithm is described in Algorithm \ref{alg1} in Appendix A. 
Since the linearity is assumed between variables, Pearson's correlation test \cite{Kowalski1972} can be used 
to test the marginally independence, $x_i \indep x_j$ and $v_i \indep v_j$. 
For testing the independence of independent variables and residuals, Shimizu and Maeda \cite{Maeda2020, Maeda2022} used the Hilbert-Schmidt independence criterion (HSIC) \cite{Gretton2007}. 
This paper uses a Kernel-based conditional independence (KCI) test \cite{Zhang2012kernel} because the accuracy of the test remains the same, and the computation time is faster.

Simply checking the conditions of Proposition \ref{Prop:1} or Proposition \ref{Prop:2} may yield ancestor relationships such that the estimated graph contains directed cycles due to errors in testing.
The proposed method adds a heuristic procedure to avoid directed cycles.
See Algorithm \ref{alg1} in the Appendix for details.

\subsection{Grouping Variables and Merging Results}
This subsection introduces a procedure for grouping $X$ into several subsets according to the estimated ancestral relationships among variables and merging the estimated causal DAGs for each group. 

Let $\widehat{Anc}_i$ be the estimated ancestor set of $x_i$. 
Define ${\cal V}$ by the output of the following procedures.
\begin{enumerate}
    \item Let ${\cal V}_0$ be the family of maximal sets in 
        $$\left\{ \{x_i\} \cup \widehat{Anc}_i, \; i=1,\ldots,p\right\}$$
    \item[2a)] If $|\bm{v}| > 1$ for all $\bm{v} \in {\cal V}_0$, then ${\cal V} \leftarrow {\cal V}_0$
    \item[2b)] Otherwise, ${\cal V} \leftarrow {\cal V}_0$\\
        For all $\bm{v} \in {\cal V}_0$ with $|\bm{v}|=1$
        $$
        {\cal V} \leftarrow 
        ({\cal V} \setminus \bm{v})
        \cup 
        \left\{
        \bigcup_{\bm{v}^\prime \in {\cal V}_0, \bm{v}^\prime \ne \bm{v}} 
        \{\bm{v} \cup \bm{v}^\prime\} 
        \right\}
        $$ 
    \item[3)] Return ${\cal V}$
\end{enumerate}

\begin{definition}
    \label{def:grouping}
    Define ${\cal V}$ as the grouping of $X$ obtained by the ancestral relationships among $X$.
\end{definition}

We call the series of procedures for obtaining ${\cal V}$ the causal ancestral-relationship-based grouping (CAG).
If the ancestral relationships among $X$ are correctly estimated and ${\cal V}={\cal V}_0$, each element of $\mathcal{V}$ is the union of a sink node in $G$ and its ancestors.
We also note that $\bigcup_{\bm{v} \in {\cal V}}  \bm{v} = X$.
Now, we have the following theorem.
\begin{theorem}
    \label{thm:group}
    Assume that we know the correct ancestral relationships between variables. 
    Define $G_{\bm{v}}$ be the sub-DAG of $G$ induced by $\bm{v} \in {\cal V}$.
    Then, the marginal model for $\bm{v}$ is the LiNGAM defined by $G_{\bm{v}}$. 
\end{theorem}

\begin{proof}
    The probability density function of $X$ is written by 
    \begin{align*}
        p(\bm{x}) &= \prod_{i=1}^p p(x_i \mid Pa_i)\\
        &= \prod_{i:x_i \in \bm{v}} p(x_i \mid Pa_i) \cdot 
        \prod_{j:x_j \notin \bm{v}} p(x_j\mid Pa_j).
    \end{align*}
    We note that when $x_i \in \bm{v}$, $Pa_i$ does not include the variables $x_j \notin \bm{v}$. 
    By integrating $p(\bm{x})$ out according to the reverse causal order of $X \setminus \bm{v}$, 
    we have
    \[
    p(\bm{v}) = \prod_{i:x_i \in \bm{v}} p(x_i \mid Pa_i), 
    \]
    which is the LiNGAM defined by $G_{\bm{v}}$.
\end{proof}

Theorem \ref{thm:group} claims that we can consistently estimate the sub-DAGs $G_{\bm{v}}$ for $\bm{v} \in {\cal V}$ by applying a causal structure learning algorithm such as DirectLiNGAM to $\bm{v}$.
Once we obtain the estimates $\hat{G}_{\bm{v}}=(\bm{v},E_{\bm{v}})$, where $E_{\bm{v}} = (\bm{v} \times \bm{v}) \cap E$, the entire causal graph $G$ can be estimated by $\hat{G}=(X, \bigcup_{\bm{v} \in {\cal V}} E_{\bm{v}})$. 

When the sample size is small, many ancestral relationships cannot be detected due to the type II error of the CI test in the CAG, so there may be many variables for which the ancestor set is empty, resulting in many groups consisting of a single variable in ${\cal V}_0$. 
Using such grouping would make the causal DAG estimates too sparse. 
So here, the group consisting of a single variable in ${\cal V}_0$ is merged with the other groups according to the above procedure 2b.
Note that Theorem \ref{thm:group} holds even if ${\cal V}_0$ contains a group consisting of one variable.


Table \ref{tab:3} presents the correct $Anc_i$ and $\bm{v}_i$ for $i=1,\ldots,9$ of the DAG in Figure \ref{Example_Fig}.
In this case, we can easily see that ${\cal V}_0={\cal V}=\{\bm{v}_3,\bm{v}_5,\bm{v}_6,\bm{v}_9\}$. 
By applying the causal structure learning algorithm to each element of ${\cal V}$, we can consistently estimate the sub-DAGs as shown in Fig. \ref{fig:group}.
By merging the sub-DAGs as $(X, E_3\cup E_5 \cup E_6 \cup E_9)$, we can obtain the true causal DAG in Figure 1.

\begin{table}[!bth]
\centering
\caption{$Anc_i$ and $\bm{v}_i$ of the DAG in Fig. 1.}
\label{tab:3}
\setlength{\tabcolsep}{7mm}{
\begin{tabular}{c|c|c}
    \hline
    Variables & $Anc_i$ & $\bm{v}_i$ \\ \hline
    $x_1$ & $\phi$ & $x_1$ \\ \hline
    $x_2$ & $x_1$ & $x_1, x_2$ \\ \hline
    $x_3$ & $x_1, x_2$ & $x_1, x_2, x_3$ \\ \hline
    $x_4$ & $x_1, x_7$ & $x_1, x_4, x_7$ \\ \hline
    $x_5$ & $x_1, x_4, x_7$ & $x_1, x_4, x_5, x_7$ \\ \hline
    $x_6$ & $x_1, x_4, x_7$ & $x_1, x_4, x_5, x_6$ \\ \hline
    $x_7$ & $\phi$ & $x_7$ \\ \hline
    $x_8$ & $x_7$ & $x_7, x_8$ \\ \hline
    $x_9$ & $x_7, x_8$ & $x_7, x_8, x_9$ \\ 
    \hline
\end{tabular}
}
\end{table}

\begin{figure}[!ht]
    \centering
    \includegraphics[width=0.45\textwidth]{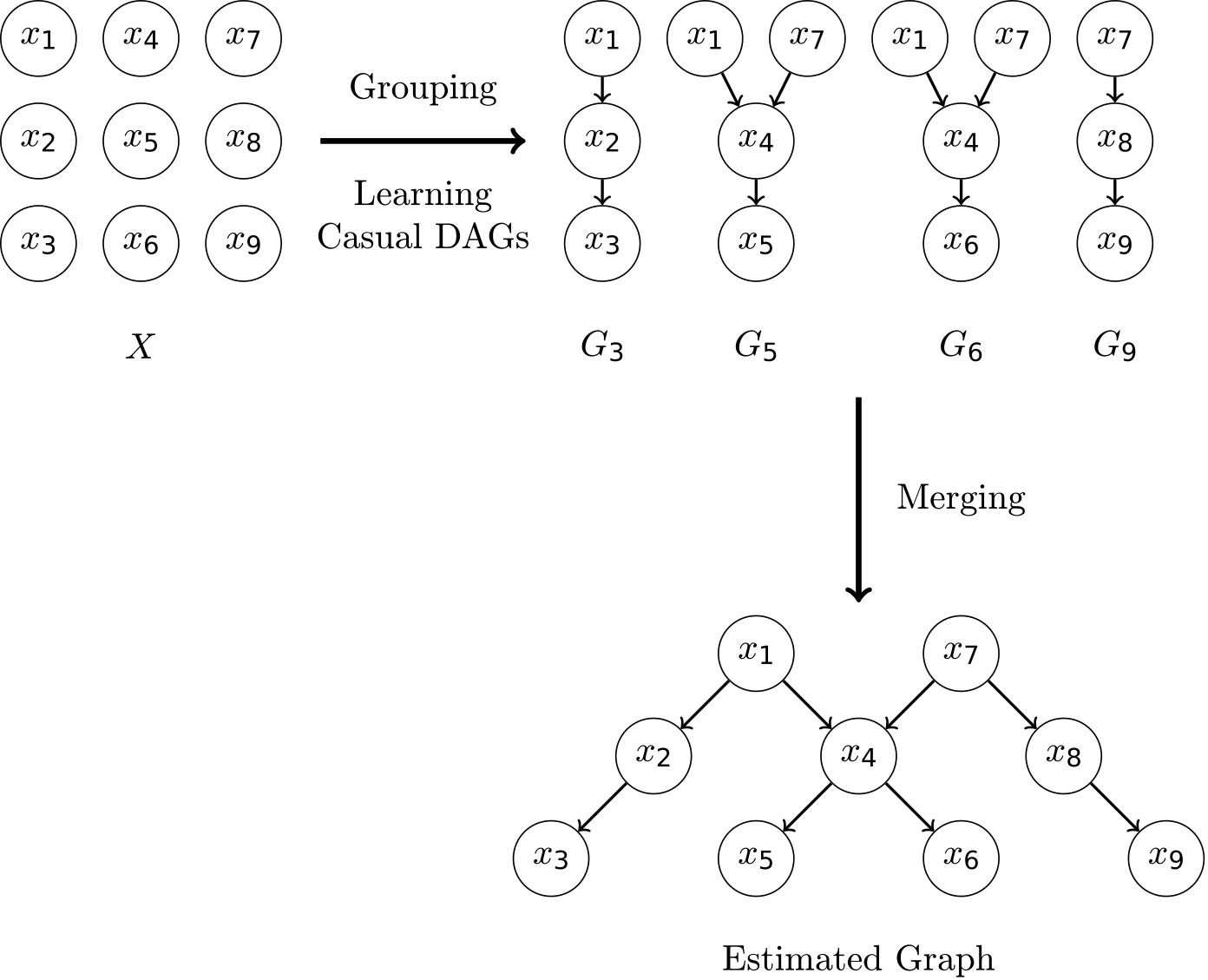}
    \caption{An example of the process of the CAG}
    \label{fig:group}
\end{figure}

In the same way as SADA, the estimated entire DAG may contain cycles due to errors in learning DAGs for each group, even if the correct ancestral relationships are estimated.
The proposed method implements topological sorting on the merged graph to check if a directed cycle exists in the merged graph. 
In SADA, the significance of each coefficient $b_{ij}$ for all edges in a cycle is assessed by the Wald test, and the least significant edge in the cycle is removed from the cycle. 
In the proposed method, since the ancestral relationships between variables are estimated, directed cycles can also be eliminated by removing edges in the cycle that contradict the estimated ancestral relationships. 
Another possible approach would be removing the edge in the cycle with the smallest absolute value of coefficient $|b_{ij}|$. The computer experiment in Section \ref{sec:4} compares the performance of these methods for removing cycles. 

\subsection{Time Complexity of Variable Grouping and Minimum Required Sample Size}
It suffices to focus on the number of CI tests to evaluate the time complexity of variable grouping by CAG. 
The case that requires the most CI tests is when the true causal DAG is fully connected, i.e., $Anc_i=Pa_i$ for all $x_i \in X$. 
In that case, since the proposed algorithm performs CI tests while removing a source node one by one from a DAG, the number of required CI tests is 
\begin{align*}          
    &
    \binom{p}{2} + \binom{p-1}{2} + \cdots + \binom{2}{2} =O(p^3).
\end{align*}
As mentioned in Section \ref{sec:LiNGAM}, the time complexity of DirectLiNGAM is $O(np^3M^2 + p^4M^3)$.
Even allowing for the time complexity of grouping variables, if $n$ is small relative to $p$, CAG is expected to reduce overall computation time compared to the original DirectLiNGAM without variable grouping.

Let $n_{\max}$ be the largest cardinality of $\bm{v} \in {\cal V}$. 
Then, the required sample size is reduced to $n_{\max} + 1$. 
\section{Computer experiments}
\label{sec:4}
This section details the results of computer experiments to compare the estimation accuracy and computation time of CAG with the original DirectLiNGAM without variable grouping, RCD in the absence of latent confounders, and CAPA with $\sigma_{\max} = 0, 1, 2$. In these experiments, CAG and CAPA used DirectLiNGAM as the causal structure learning algorithm for each variable group. 
In the following, we will refer to these procedures as CAG-LiNGAM and CAPA-LiNGAM, respectively. The accuracy of CAG-LiNGAM using true ancestral relationships was also computed for reference. 
As mentioned in Section \ref{sec:RCD}, RCD can also be applied to the model in the presence of latent confounders. This paper assumes that no latent confounder exists, so in this experiment for RCD, the procedure for determining latent confounders was omitted to estimate ancestral and parent-child relationships among variables.


\subsection{Experimental Settings}
The number of variables $p$ and the number of edges $|E|$ were set to 
$(p, |E|) = (10, 5)$, $(10, 9)$, $(20, 19)$, $(30, 29)$ and $(40, 39)$.
The sample sizes $n$ were set to $11, 25, 50, 100$ and $500$.

The error terms $e_i$ were generated from the uniform distribution $U(-1,1)$.
The true causal DAGs were randomly generated with $(p,|E|)$ fixed at each iteration. 
In these experiments, the true causal DAGs were assumed to be sparse, and the indegrees of all vertices were set to at most one.
The nonzero elements of the coefficient matrix $B$ were randomly generated from the uniform distribution $U([-1, -0.5] \cup [0.5, 1])$ at each iteration.
The number of iterations was set to 100. 

Let $|E_c|$ be the number of correctly detected edges in the 100 estimated DAGs. 
Denote the number of redundant and missing edges in the 100 estimated DAGs by $|E_{r}|$ and $|E_{m}|$, respectively.
Note that if an edge in an estimated DAG is in the wrong direction, both $|E_r|$ and $|E_m|$ will add one.
We evaluated the performance of CAG-LiNGAM and the other methods using the following indices 
(e.g., Zhu et al. \cite{Zhu2019}). 

\begin{enumerate}
    \item The precision 
        \[
        Pre:=\frac{|E_c|}{|E_c| + |E_r|}
        \]              
        is the ratio of correctly estimated edges among the edges in the estimated DAG; 
    \item The recall 
        \[
        Rec:=\frac{|E_c|}{|E_c| + |E_m|} = \frac{|E_c|}{|E|}
        \]        
        is the ratio of correctly estimated edges among the edges in the true DAG; 
    \item $F$-measure
        \[
        F:=\frac{2 \times Pre \times Rec}{Pre + Rec}
        \]
        is the harmonic mean of the precision and the recall; 
    \item $Time$ is the running time in seconds for estimating 100 DAGs.
\end{enumerate}

Let $\alpha_P$ be the significance level of Pearson's correlation test in CAG and RCD.
Let $\alpha_{CI}$ be the significance level of KCI in CAG, RCD, and CAPA.
$\alpha_P$ and $\alpha_{CI}$ are set to 
\begin{align*}
    \alpha_P &\in \{0.01, 0.05, 0.1, 0.5\},\\
    \alpha_{CI} & \in \{0.001, 0.01, 0.1, 0.2, 0.3, 0.4, 0.5\}, 
\end{align*}
respectively. 

Experiments for $p=40$ were conducted on a machine with a 3.0 GHz Intel Core i9 processor and 256 GB memory, and experiments for $p=30$ were performed on a machine with a 3.0 GHz Intel Core i9 processor and 128 GB memory. 
The other experiments were conducted on a machine with a 2.1 GHz Dual Xeon processor and 96 GB memory. 

\subsection{Experimental Results and Discussion}
\label{Experimental Results}
\begin{table*}
\centering

\caption{Performances of CAG, DirectLiNGAM, RCD and CAPA}
\label{tab:4}
\scalebox{0.7}{
\begin{tabular}{l|ccccccccc|ccccccccc} 
\hline
\multirow{3}{*}{($p$, $|E|$, $n$)} & \multicolumn{9}{c|}{$Pre$}                                                                                                                      & \multicolumn{9}{c}{$Rec$}                                                                                                             \\ 
\cline{2-19}
                                   & \multicolumn{3}{c}{CAG}                                            & $\mathrm{CAG}^*$ & DLi   & RCD            & \multicolumn{3}{c|}{CAPA}      & \multicolumn{3}{c}{CAG}                          & $\mathrm{CAG}^*$ & DLi             & RCD            & \multicolumn{3}{c}{CAPA}     \\ 
\cline{2-19}
                                   & W              & \multicolumn{1}{l}{Anc} & \multicolumn{1}{l}{abs} & Anc              &       & W              & 0     & 1     & 2              & W              & Anc            & abs            & Anc              &                 & W              & 0      & 1       & 2         \\ 
\hline
(10, 5, 11)                        & 0.210          & 0.221                   & 0.213                   & 0.507            & 0.122 & \textbf{0.268} & 0.209 & 0.166 & 0.193          & 0.434          & 0.396          & \textbf{0.446} & 0.448            & 0.374           & 0.128          & 0.197  & 0.209   & 0.194     \\
(10, 5, 25)                        & 0.550          & 0.543                   & 0.546                   & 0.720            & 0.311 & 0.602          & 0.493 & 0.597 & \textbf{0.612} & 0.576          & \textbf{0.578} & \textbf{0.578} & 0.736            & 0.576           & 0.396          & 0.483  & 0.420   & 0.408     \\
(10, 5, 50)                        & 0.773          & 0.770                   & 0.773                   & 0.850            & 0.521 & \textbf{0.851} & 0.725 & 0.735 & 0.779          & 0.764          & 0.764          & 0.764          & 0.928            & \textbf{0.868}  & 0.592          & 0.696  & 0.723   & 0.706     \\
(10, 5, 100)                       & 0.891          & 0.889                   & 0.891                   & 0.917            & 0.693 & \textbf{0.967} & 0.786 & 0.877 & 0.896          & 0.896          & 0.896          & 0.896          & 0.996            & \textbf{0.986}  & 0.826          & 0.886  & 0.860   & 0.840     \\
(10, 5, 500)                       & 0.971          & 0.971                   & 0.971                   & 0.963            & 0.813 & \textbf{1.000} & 0.861 & 0.901 & 0.906          & 0.998          & 0.998          & 0.998          & 1.000            & \textbf{1.000}  & 0.994          & 0.945  & 0.956   & 0.960     \\ 
\hline
(10, 9, 11)                        & 0.251          & 0.218                   & 0.244                   & 0.396            & 0.179 & \textbf{0.281} & 0.214 & 0.208 & 0.250          & 0.431          & \textbf{0.487} & 0.452          & 0.483            & 0.387           & 0.147          & 0.207  & 0.220   & 0.193     \\
(10, 9, 25)                        & 0.369          & 0.368                   & 0.350                   & 0.563            & 0.339 & \textbf{0.452} & 0.388 & 0.389 & 0.429          & 0.693          & 0.486          & \textbf{0.727} & 0.729            & 0.600           & 0.301          & 0.418  & 0.447   & 0.426     \\
(10, 9, 50)                        & 0.621          & 0.599                   & 0.609                   & 0.686            & 0.482 & \textbf{0.644} & 0.526 & 0.523 & 0.548          & 0.643          & 0.650          & 0.646          & 0.899            & \textbf{0.818}  & 0.521          & 0.634  & 0.651   & 0.643     \\
(10, 9, 100)                       & 0.779          & 0.768                   & 0.779                   & 0.834            & 0.685 & \textbf{0.914} & 0.680 & 0.686 & 0.682          & 0.851          & 0.860          & 0.859          & 0.986            & \textbf{0.974}  & 0.758          & 0.845  & 0.836   & 0.839     \\
(10, 9, 500)                       & 0.918          & 0.917                   & 0.918                   & 0.877            & 0.785 & \textbf{0.997} & 0.785 & 0.773 & 0.774          & 0.999          & 0.999          & 0.999          & 1.000            & \textbf{1.000}  & 0.999          & 0.929  & 0.930   & 0.926     \\ 
\hline
(20, 19, 25)                       & 0.313          & 0.319                   & 0.267                   & 0.435            & 0.172 & 0.327          & 0.235 & 0.410 & \textbf{0.468} & 0.601          & 0.423          & \textbf{0.642} & 0.665            & 0.496           & 0.284          & 0.287  & 0.220   & 0.238     \\
(20, 19, 50)                       & 0.528          & 0.484                   & 0.497                   & 0.625            & 0.354 & 0.549          & 0.447 & 0.454 & \textbf{0.616} & 0.662          & 0.640          & 0.676          & 0.876            & \textbf{0.721}  & 0.514          & 0.561  & 0.527   & 0.448     \\
(20, 19, 100)                      & 0.759          & 0.723                   & 0.751                   & 0.780            & 0.576 & \textbf{0.837} & 0.610 & 0.594 & 0.681          & 0.776          & 0.780          & 0.778          & 0.984            & \textbf{0.943}  & 0.729          & 0.779  & 0.776   & 0.685     \\
(20, 19, 500)                      & 0.897          & 0.894                   & 0.897                   & 0.844            & 0.725 & \textbf{0.994} & 0.725 & 0.718 & 0.712          & 0.997          & 0.997          & 0.997          & 1.000            & \textbf{1.000}  & 0.996          & 0.899  & 0.895   & 0.892     \\ 
\hline
(30, 29, 50)                       & 0.535          & 0.445                   & 0.489                   & 0.655            & 0.276 & 0.475          & 0.374 & 0.527 & \textbf{0.631} & 0.618          & 0.616          & \textbf{0.633} & 0.855            & 0.588           & 0.479          & 0.438  & 0.390   & 0.435     \\
(30, 29, 100)                      & 0.725          & 0.701                   & 0.673                   & 0.878            & 0.587 & \textbf{0.813} & 0.649 & 0.649 & 0.676          & 0.725          & 0.701          & \textbf{0.853} & 0.878            & 0.587           & 0.813          & 0.649  & 0.649   & 0.676     \\
(30, 29, 500)                      & 0.969          & 0.967                   & 0.970                   & 0.976            & 0.919 & \textbf{0.998} & 0.919 & 0.916 & 0.914          & 0.998          & 0.998          & 0.998          & 1.000            & \textbf{1.000}  & 0.996          & 0.970  & 0.968   & 0.969     \\ 
\hline
(40, 39, 50)                       & 0.464          & 0.423                   & 0.466                   & 0.625            & 0.211 & 0.459          & 0.311 & 0.480 & \textbf{0.712} & \textbf{0.800} & 0.595          & 0.624          & 0.839            & 0.536           & 0.469          & 0.381  & 0.364   & 0.374     \\
(40, 39, 100)                      & 0.696          & 0.640                   & 0.652                   & 0.863            & 0.477 & \textbf{0.792} & 0.573 & 0.576 & 0.679          & 0.835          & 0.741          & \textbf{0.849} & 0.973            & 0.809           & 0.686          & 0.690  & 0.680   & 0.635     \\
(40, 39, 500)                      & 0.952          & 0.944                   & 0.951                   & 0.970            & 0.917 & \textbf{0.991} & 0.917 & 0.914 & 0.915          & 0.998          & 0.998          & 0.998          & 1.000            & \textbf{1.000}  & 0.995          & 0.966  & 0.966   & 0.966     \\ 
\hline
\multirow{3}{*}{($p$, $|E|$, $n$)} & \multicolumn{9}{c|}{$F$}                                                                                                                        & \multicolumn{9}{c}{$Time(s)$}                                                                                                         \\ 
\cline{2-19}
                                   & \multicolumn{3}{c}{CAG}                                            & $\mathrm{CAG}^*$ & DLi   & RCD            & \multicolumn{3}{c|}{CAPA}      & \multicolumn{3}{c}{CAG}                          & $\mathrm{CAG}^*$ & DLi             & RCD            & \multicolumn{3}{c}{CAPA}     \\ 
\cline{2-19}
                                   & W              & \multicolumn{1}{l}{Anc} & \multicolumn{1}{l}{abs} & Anc              &       & W              & 0     & 1     & 2              & W              & Anc            & abs            & Anc              &                 & W              & 0      & 1       & 2         \\ 
\hline
(10, 5, 11)                        & 0.283          & 0.284                   & \textbf{0.288}          & 0.476            & 0.184 & 0.173          & 0.203 & 0.185 & 0.194          & 23.0           & 17.8           & 21.3           & 5.0              & 29.0            & \textbf{6.7}   & 28.2   & 99.5    & 210.7     \\
(10, 5, 25)                        & \textbf{0.563} & 0.560                   & 0.562                   & 0.728            & 0.404 & 0.478          & 0.488 & 0.493 & 0.489          & 13.5           & 13.4           & 13.4           & 5.0              & 27.5            & \textbf{8.0}   & 32.2   & 55.0    & 57.1      \\
(10, 5, 50)                        & \textbf{0.769} & 0.767                   & \textbf{0.769}          & 0.887            & 0.651 & 0.698          & 0.710 & 0.729 & 0.741          & 13.0           & 13.0           & 13.0           & 5.1              & 72.1            & \textbf{8.1}   & 32.8   & 75.1    & 84.8      \\
(10, 5, 100)                       & \textbf{0.893} & 0.892                   & \textbf{0.893}          & 0.955            & 0.814 & 0.891          & 0.833 & 0.868 & 0.867          & 154.6          & 154.4          & 154.5          & 48.9             & 123.6           & \textbf{123.4} & 165.2  & 371.7   & 502.2     \\
(10, 5, 500)                       & 0.984          & 0.984                   & 0.984                   & 0.981            & 0.897 & \textbf{0.997} & 0.901 & 0.928 & 0.932          & 330.8          & 330.7          & 330.8          & 65.1             & \textbf{151.8}  & 480.5          & 459.5  & 1867.9  & 2529.8    \\ 
\hline
(10, 9, 11)                        & \textbf{0.317} & 0.301                   & \textbf{0.317}          & 0.435            & 0.244 & 0.193          & 0.210 & 0.214 & 0.218          & 39.4           & 23.4           & 23.5           & 12.3             & 28.4            & \textbf{10.2}  & 39.0   & 112.3   & 250.8     \\
(10, 9, 25)                        & \textbf{0.482} & 0.419                   & 0.472                   & 0.635            & 0.433 & 0.361          & 0.403 & 0.416 & 0.428          & 34.1           & 36.5           & 25.0           & 11.3             & 28.9            & \textbf{15.2}  & 44.2   & 110.8   & 291.0     \\
(10, 9, 50)                        & \textbf{0.632} & 0.623                   & 0.627                   & 0.778            & 0.607 & 0.576          & 0.575 & 0.580 & 0.591          & 22.8           & 21.7           & 21.7           & 12.3             & 71.8            & \textbf{18.6}  & 80.3   & 187.3   & 370.5     \\
(10, 9, 100)                       & 0.814          & 0.811                   & 0.817                   & 0.904            & 0.805 & \textbf{0.829} & 0.753 & 0.753 & 0.752          & 338.9          & 334.6          & 334.6          & 117.1            & \textbf{120.7}  & 312.3          & 255.9  & 1420.2  & 6078.1    \\
(10, 9, 500)                       & 0.957          & 0.956                   & 0.957                   & 0.935            & 0.879 & \textbf{0.998} & 0.851 & 0.844 & 0.843          & 1062.1         & 1061.6         & 1061.6         & 154.4            & \textbf{155.5}  & 1332.3         & 425.2  & 6833.6  & 19649.6   \\ 
\hline
(20, 19, 25)                       & \textbf{0.411} & 0.364                   & 0.378                   & 0.526            & 0.255 & 0.304          & 0.259 & 0.286 & 0.316          & 520.0          & 47.2           & 101.7          & 29.2             & 172.3           & \textbf{34.4}  & 524.3  & 713.6   & 1044.2    \\
(20, 19, 50)                       & \textbf{0.588} & 0.551                   & 0.573                   & 0.729            & 0.474 & 0.531          & 0.498 & 0.488 & 0.518          & 93.2           & 57.0           & 72.2           & 29.3             & 323.9           & \textbf{48.6}  & 573.4  & 926.4   & 1658.1    \\
(20, 19, 100)                      & 0.768          & 0.750                   & 0.764                   & 0.870            & 0.715 & \textbf{0.779} & 0.684 & 0.673 & 0.683          & 1111.4         & 1015.5         & 1015.7         & 337.2            & \textbf{417.3}  & 991.4          & 1418.6 & 10457.6 & 39708.9   \\
(20, 19, 500)                      & 0.944          & 0.943                   & 0.944                   & 0.916            & 0.840 & \textbf{0.995} & 0.802 & 0.797 & 0.792          & 3918.8         & 3906.2         & 3906.3         & 421.7            & \textbf{542.6}  & 4757.2         & 1471.6 & 33655.4 & 166854.1  \\ 
\hline
(30, 29, 50)                       & \textbf{0.573} & 0.516                   & 0.552                   & 0.742            & 0.376 & 0.477          & 0.403 & 0.449 & 0.515          & 196.6          & 92.7           & 129.2          & 40.2             & 554.7           & \textbf{65.8}  & 1077.6 & 1962.0  & 5290.4    \\
(30, 29, 100)                      & \textbf{0.782} & 0.726                   & 0.753                   & 0.925            & 0.703 & 0.747          & 0.704 & 0.703 & 0.709          & 1381.5         & 1138.7         & 1215.6         & 397.6            & \textbf{604.9}  & 1287.1         & 1090.2 & 13253.2 & 117849.3  \\
(30, 29, 500)                      & 0.983          & 0.982                   & 0.984                   & 0.988            & 0.958 & \textbf{0.997} & 0.943 & 0.941 & 0.941          & 5732.6         & 5731.6         & 5731.7         & 444.0            & \textbf{1449.6} & 9536.9         & 3540.0 & 45572.6 & 341942.2  \\ 
\hline
(40, 39, 50)                       & \textbf{0.588} & 0.494                   & 0.533                   & 0.716            & 0.302 & 0.464          & 0.343 & 0.414 & 0.490          & 4233.4         & 192.9          & 258.3          & 56.9             & 962.4           & \textbf{153.3} & 1920.0 & 4387.8  & 6462.0    \\
(40, 39, 100)                      & \textbf{0.759} & 0.687                   & 0.737                   & 0.914            & 0.600 & 0.735          & 0.626 & 0.624 & 0.656          & 2931.3         & 2527.1         & 2620.8         & 604.1            & \textbf{1001.2} & 2725.9         & 1740.4 & 35769.3 & 352156.9  \\
(40, 39, 500)                      & 0.975          & 0.970                   & 0.974                   & 0.985            & 0.956 & \textbf{0.993} & 0.941 & 0.939 & 0.940          & 10234.1        & 10225.4        & 10225.7        & 706.8            & \textbf{2610.9} & 16099.2        & 6221.7 & 94989.2 & 779946.2  \\
\hline
\end{tabular}
}
\end{table*}

Table \ref{tab:4} shows $Pre$, $Rec$, $F$ and $Time$ for each method for some $(p,|E|,n)$. 
The experiments for CAG-LiNGAM and RCD were conducted with $4 \times 7$ combinations of significance levels of Pearson's correlation test and KCI for each $(p,|E|,n)$. 
The experiments for CAPA-LiNGAM were conducted with $7$ significance levels of KCI for each
$(p,|E|,n)$. 
Table \ref{tab:4} shows only the results of these methods for the significance level that maximizes $F$-measure.

In Table \ref{tab:4}, CAG represents CAG-LiNGAM, while $\text{CAG}^*$ refers to CAG-LiNGAM using true ancestral relationships. 
DLi stands for the original DirectLiNGAM without grouping variables. 
W, Anc, and abs under CAG and RCD represent cycle elimination methods based on the Wald test, the estimated ancestral relationship, and the absolute value of the coefficient, respectively.
The numbers 0, 1, and 2 under CAPA are the $\sigma_{\max}$ values. 
In the table, the bold numbers highlight the best performance in each experimental group $(p, |E|, n)$, excluding $\text{CAG}^*$. 


As seen from Table \ref{tab:4}, CAG-LiNGAM outperforms the original DirectLiNGAM without the variable grouping, RCD, and CAPA-LiNGAM when the sample size is small relative to the number of variables in terms of $F$-measure and the recall.
Although CAG requires $O(p^3)$ time complexity to group variables, CAG-LiNGAM often shows shorter computation time than the original DirectLiNGAM when the sample size is small. 
When the sample size is large, CAG-LiNGAM takes longer computation time than the original DirectLiNGAM, but CAG-LiNGAM is superior to the original DirectLiNGAM in terms of precision and $F$-measure.

CAG with Wald tests takes longer computation time but is more accurate in terms of $F$-measure than CAG with Anc and abs. As the sample size increases, the accuracy of estimating ancestral relationships and DirectLiNGAM increases, and hence, the frequency of cycles appearing in the estimated DAG decreases, so the difference in estimation accuracy between the three CAG variants becomes smaller.

When the sample size is small, CAPA-LiNGAM may show higher precision than CAG-LiNGAM. However, the recall and $F$-measure are lower in many cases.
This may be because when the sample size is small, it is difficult for CAPA to detect conditional independence relationships due to the type II error of the CI tests. Then, $G_\sigma$ becomes overly sparse, resulting in overly sparse causal DAG estimates. 

When $\sigma_{\max}$ is $0$ and $1$, the time complexity of the variable grouping of CAPA is less than or equal to that of CAG. 
However, the computation time for CAPA-LiNGAM is not much different from that for CAG-LiNGAM, or CAG-LiNGAM is faster than CAPA-LiNGAM. Besides, CAG-LiNGAM outperforms CAPA-LiNGAM in terms of $F$-measure.

It is noteworthy that RCD shows higher $F$-measure than CAG-LiNGAM when the sample size is large, although it takes more computation time. 
Conversely, when the sample size is small, CAG-LiNGAM shows higher $F$-measure than RCD, although it takes more computation time. 

The original DirectLiNGAM shows lower precision than divide-and-conquer algorithms. 
This may be because the original DirectLiNGAM does not group variables, leading to redundant edges across different variable groups. 
In general, DAGs estimated by DirectLiNGAM tend to have more redundant edges.

Even if the true causal DAG is sparse, the subgraph induced by each variable group becomes relatively denser. 
In the case of directed trees, the ratio of the number of edges to the number of variable pairs is 
$(p-1)/{}_pC_2=2/p$. 
This means that the smaller the number of variables in each group is, the denser the sub-DAG induced by each group is.
When variables are correctly grouped into smaller subsets, the true sub-DAGs induced by each group become denser, and there are no edges across groups. Therefore, if ancestral relationships can be accurately estimated, applying DirectLiNGAM to each group is expected to avoid more redundant edges than using it for the entire variable. 

The CAG uses the estimated ancestral relationships for grouping variables but does not use them to estimate the causal DAG for each group.
However, since ancestral relationships partially define the structure of the causal DAG, information on ancestral relationships could be used to estimate causal DAGs.
The RCD can be interpreted as using information on ancestral relationships to estimate causal DAGs.
\begin{figure}[!ht]
    \centering
    \scalebox{0.5}{
        \includegraphics{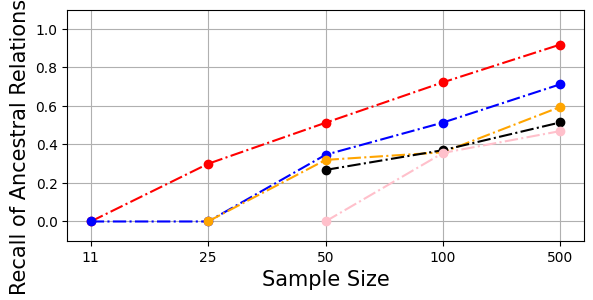}
    }
    \caption{$Anc_{rec}$ with Wald tests for some sample sizes}
    {\footnotesize The lines in $\lbrace$red, blue, orange, black, pink$\rbrace$ represent
    $(p, |E|) = (10, 5), (10, 9), (20, 19), (30. 29), (40, 39)$, respectively.}
    \label{ANC}
\end{figure}
Let $A_c$ and $A_m$ be the number of correctly detected and missing ancestors in the estimated ancestor lists
$\widehat{Anc}_i$, $i=1,\ldots,p$, respectively. 
Define $Anc_{rec}$ by the recall of the list of ancestors, 
\[
Anc_{rec} = \frac{A_c}{A_c+A_m}. 
\]
Figure \ref{ANC} plots $Anc_{rec}$ of the ancestor list for some $(p,|E|,n)$ when using the significance levels of the CI tests that maximize the $F$-measure of the CAG-LiNGAM with Wald tests. 
This figure shows that when the sample size is small, the accuracy of the estimation of ancestor lists is quite low. 
When $(p,|E|,n)=(10,5,11), (10,9,11), (10,9,25), (20,19,25), (40, 39, 50)$, $Anc_{rec}$ are exactly zeros. 
However, Table \ref{tab:4} shows that the accuracy of CAG-LiNGAM is not that bad.
This fact suggests that applying DirectLiNGAM to each group split by the CAG may correct errors in the estimated ancestral relationships within each group.

Table \ref{tab:6} shows $Anc_{rec}$ computed from the ancestral relationships estimated by CAG and $Anc_{rec}$ computed from the estimated DAG in the experiments. 
We can see that $Anc_{rec}$ is significantly improved by applying DirectLiNGAM to each group. 

\begin{table}[!ht]
\centering
\caption{$Anc_{rec}$ before and after applying DirectLiNGAM}
\label{tab:6}
\scalebox{1}{
\begin{tabular}{l|ccc|ccc} \hline
              & \multicolumn{3}{c|}{Before} & \multicolumn{3}{c}{After}  \\ 
\cline{2-7}
$(p, |E|, n)$   &   W     & Anc   & abs    & W     & Anc   & abs        \\ 
\hline
(10, 5, 11)   & 0     & 0.049 & 0     & 0.438 & 0.405 & 0.466      \\
(10, 5, 25)   & 0.300 & 0.300 & 0.300 & 0.519 & 0.525 & 0.524      \\
(10, 5, 50)   & 0.513 & 0.513 & 0.513 & 0.677 & 0.679 & 0.679      \\
(10, 5, 100)  & 0.723 & 0.723 & 0.723 & 0.852 & 0.852 & 0.852      \\
(10, 5, 500)  & 0.919 & 0.919 & 0.919 & 0.998 & 0.998 & 0.998      \\ 
\hline
(10, 9, 11)   & 0     & 0     & 0     & 0.404 & 0.568 & 0.447      \\
(10, 9, 25)   & 0     & 0.323 & 0     & 0.648 & 0.483 & 0.701      \\
(10, 9, 50)   & 0.346 & 0.346 & 0.346 & 0.540 & 0.557 & 0.547      \\
(10, 9, 100)  & 0.514 & 0.514 & 0.514 & 0.791 & 0.817 & 0.805      \\
(10, 9, 500)  & 0.713 & 0.713 & 0.713 & 0.999 & 0.999 & 0.999      \\ 
\hline
(20, 19, 25)  & 0.001 & 0.177 & 0.001 & 0.514 & 0.366 & 0.577      \\
(20, 19, 50)  & 0.320 & 0.289 & 0.320 & 0.545 & 0.558 & 0.594      \\
(20, 19, 100) & 0.359 & 0.359 & 0.359 & 0.638 & 0.669 & 0.651      \\
(20, 19, 500) & 0.596 & 0.596 & 0.596 & 0.993 & 0.998 & 0.998      \\ 
\hline
(30, 29, 50)  & 0.268 & 0.241 & 0.268 & 0.450 & 0.562 & 0.515      \\
(30, 29, 100) & 0.371 & 0.291 & 0.371 & 0.745 & 0.636 & 0.782      \\
(30, 29, 500) & 0.515 & 0.515 & 0.515 & 0.993 & 0.997 & 0.993      \\ 
\hline
(40, 39, 50)  & 0.001 & 0.222 & 0.225 & 0.627 & 0.546 & 0.508      \\
(40, 39, 100) & 0.356 & 0.276 & 0.356 & 0.697 & 0.624 & 0.785      \\
(40, 39, 500) & 0.470 & 0.470 & 0.470 & 0.997 & 0.994 & 0.997     \\ \hline
\end{tabular}
}
\end{table}

Table \ref{tab:Anc} summarizes the following six indices calculated from CAG's experimental results 
when $p=10$. 
\begin{itemize}
    \item Improvement ($Imp$): the number of ancestral relationships in the output DAGs that DirectLiNGAM corrects
    \item Worsening ($Wor$): the number of ancestral relationships in the output DAGs that DirectLiNGAM worsens
    \item Errors in estimation ($Err$): the number of ancestral relationships that are wrongly estimated in CAG
    \item Correct estimation ($Corr$): the number of ancestral relationships that are correctly estimated in CAG
    \item $Imp/Err$:  the ratio of ancestral relationships incorrectly estimated by CAG that DirectLiNGAM corrected
    \item $Wor/Corr$: the ratio of ancestral relationships correctly estimated by CAG that DirectLiNGAM worsened
\end{itemize}

\begin{table}
\centering
\caption{The Improvement and Worsening of Ancestral Relationships after applying DirectLiNGAM.}
\label{tab:Anc}
\scalebox{0.75}{
\begin{tabular}{l|c|cccc|cc} \hline
$(p, |E|, n)$                   & Elimination & $Imp$ & $Wor$ & $Corr$ & $Err$ & $Imp/Err$ & $Wor/Corr$              \\ 
\hline
\multirow{3}{*}{(10,5,11)}  & Wald        & 265 & 1029 & 3858 & 642  & \textbf{0.4128} & 0.2667           \\
                            & Anc         & 228 & 858  & 3858 & 642  & \textbf{0.3551}  & 0.2224           \\
                            & abs       & 286 & 1042 & 3858 & 642  & \textbf{0.4455} & 0.2701           \\ 
\hline
\multirow{3}{*}{(10,5,25)}  & Wald        & 239 & 700  & 3956 & 544  & \textbf{0.4393} & 0.1769           \\
                            & Anc         & 227 & 636  & 3956 & 544  & \textbf{0.4173} & 0.1608           \\
                            & abs       & 238 & 720  & 3956 & 544  & \textbf{0.4375}   & 0.1820           \\ 
\hline
\multirow{3}{*}{(10,5,50)}  & Wald        & 204 & 268  & 4154 & 346  & \textbf{0.5896} & 0.0645           \\
                            & Anc         & 191 & 253  & 4154 & 346  & \textbf{0.5520} & 0.0609           \\
                            & abs       & 211 & 280  & 4154 & 346  & \textbf{0.6098} & 0.0674           \\ 
\hline
\multirow{3}{*}{(10,9,11)}  & Wald        & 794 & 1226 & 2627 & 1873 & 0.4239          & \textbf{0.4667}  \\
                            & Anc         & 374 & 559  & 2627 & 1873 & 0.1997           & \textbf{0.2128}   \\
                            & abs       & 794 & 1226 & 2627 & 1873 & 0.4239          & \textbf{0.4667}  \\ 
\hline
\multirow{3}{*}{(10,9,25)}  & Wald        & 657 & 678  & 2780 & 1720 & \textbf{0.3820} & 0.2439           \\
                            & Anc         & 519 & 565  & 2780 & 1720 & \textbf{0.3017} & 0.2032           \\
                            & abs       & 747 & 804  & 2780 & 1720 & \textbf{0.4343} & 0.2892           \\ 
\hline
\multirow{3}{*}{(10,9,50)}  & Wald        & 704 & 373  & 3078 & 1422 & \textbf{0.4951} & 0.1212           \\
                            & Anc         & 570 & 273  & 3078 & 1422 & \textbf{0.4008} & 0.0887           \\
                            & abs       & 728 & 412  & 3078 & 1422 & \textbf{0.5120} & 0.1339           \\ 
\hline
\multirow{3}{*}{(10,9,100)} & Wald        & 682 & 130  & 3479 & 1021 & \textbf{0.6680} & 0.0374           \\
                            & Anc         & 586 & 122  & 3479 & 1021 & \textbf{0.5739} & 0.0351           \\
                            & abs       & 703 & 118  & 3479 & 1021 & \textbf{0.6885} & 0.0339           \\ 
\hline
\multirow{3}{*}{(10,9,500)} & Wald        & 626 & 32   & 3730 & 770  & \textbf{0.8130} & 0.0086           \\
                            & Anc         & 626 & 33   & 3730 & 770  & \textbf{0.8078} & 0.0088           \\
                            & abs       & 626 & 32   & 3730 & 770  & \textbf{0.8104}  & 0.0086     \\ \hline     
\end{tabular}
}
\end{table}

Table \ref{tab:Anc} shows that when the sample size is small, $Imp < Wor$, but $Imp/Err > Wor/Corr$. By applying DirectLiNGAM, the number of corrected ancestral relationships is less than that of worsened ones. 
However, the proportion of corrected ancestral relationships is relatively larger than that of worsened ancestral relationships, which may have improved the accuracy of causal DAG estimation.

Even with large sample sizes, $Anc_{rec}$ improves by applying DirectLiNGAM. However, Table \ref{tab:4} shows that the accuracy of CAG is inferior to that of RCD. When the sample size is large, DirectLiNGAM tends to include redundant edges in the estimated DAG, which may make CAG less accurate than RCD.

This experiment set 28 different significance levels for the CAG's CI tests. 
Table \ref{tab:alpha} presents the significance levels that maximize the $F$-measure of CAG-LiNGAM.
A small significance level is often chosen when the sample size is small. In practice, however, the results are almost the same regardless of the significance level chosen.
When the sample size is small, ancestral relationships can hardly be detected regardless of the significance level. In this experiment of CAG-LiNGAM with Wald tests, when $(p,|E|,n)=(10,5,11), (10,9,11),(10,9,25)$, no ancestral relationship could be detected within 100 repetitions. From Definition \ref{def:grouping} in Section \ref{sec:ancestor}, if no ancestral relationship was detected, the variable group is the family of all $_pC_2$ variable pairs. Interestingly, even in such cases, CAG-LiNGAM shows higher $F$-measure than the other methods.
When the sample size is moderately small, the $F$-measure values are large when the significance level of Pearson's correlation test is 0.01 and the significance level of KCI is about 0.2 to 0.5.
When the sample size is large, the smaller the significance level, the higher $F$-measure.

\begin{table}[H]
\centering
\caption{The significance levels used in CAG-LiNGAM in Table \ref{tab:4}}
\label{tab:alpha}
\scalebox{1}{
\begin{tabular}{l|cc|cc|cc} \hline
              & \multicolumn{2}{c|}{W} & \multicolumn{2}{c|}{Anc} & \multicolumn{2}{c}{abs}  \\ 
\cline{2-7}
($p$, $|E|$, $n$) & $\alpha_P$ & $\alpha_{CI}$ & $\alpha_{P}$    & $\alpha_{CI}$                 & $\alpha_{P}$     & $\alpha_{CI}$                  \\ 
\hline
(10, 5, 11)   & 0.01 & 0.001           & 0.01 & 0.3               & 0.01 & 0.001             \\
(10, 5, 25)   & 0.01 & 0.5             & 0.01 & 0.5               & 0.01 & 0.5               \\
(10, 5, 50)   & 0.01 & 0.4             & 0.01 & 0.4               & 0.01 & 0.4               \\
(10, 5, 100)  & 0.01 & 0.3             & 0.01 & 0.3               & 0.01 & 0.3               \\
(10, 5, 500)  & 0.01 & 0.001           & 0.01 & 0.001             & 0.01 & 0.001             \\ 
\hline
(10, 9, 11)   & 0.05 & 0.01            & 0.05 & 0.01              & 0.05 & 0.01              \\
(10, 9, 25)   & 0.01 & 0.001           & 0.5  & 0.5               & 0.01 & 0.001             \\
(10, 9, 50)   & 0.01 & 0.3             & 0.01 & 0.3               & 0.01 & 0.3               \\
(10, 9, 100)  & 0.01 & 0.2             & 0.01 & 0.2               & 0.01 & 0.2               \\
(10, 9, 500)  & 0.01 & 0.001           & 0.01 & 0.001             & 0.01 & 0.001             \\ 
\hline
(20, 19, 25)  & 0.05 & 0.001           & 0.01 & 0.5               & 0.05 & 0.001             \\
(20, 19, 50)  & 0.01 & 0.4             & 0.01 & 0.3               & 0.01 & 0.4               \\
(20, 19, 100) & 0.01 & 0.1             & 0.01 & 0.1               & 0.01 & 0.1               \\
(20, 19, 500) & 0.01 & 0.001           & 0.01 & 0.001             & 0.01 & 0.001             \\ 
\hline
(30, 29, 50)  & 0.01 & 0.4             & 0.01 & 0.3               & 0.01 & 0.4               \\
(30, 29, 100) & 0.01 & 0.2             & 0.01 & 0.1               & 0.01 & 0.2               \\
(30, 29, 500) & 0.01 & 0.001           & 0.01 & 0.001             & 0.01 & 0.001             \\ 
\hline
(40, 39, 50)  & 0.01 & 0.001           & 0.01 & 0.3               & 0.01 & 0.4               \\
(40, 39, 100) & 0.01 & 0.2             & 0.01 & 0.1               & 0.01 & 0.1               \\
(40, 39, 500) & 0.01 & 0.001           & 0.01 & 0.001             & 0.01 & 0.001            \\ \hline
\end{tabular}
}
\end{table}

As mentioned, $\mathrm{CAG}^*$ is the CAG-LiNGAM when the true ancestral relationships are known. When the sample size is small, the estimation accuracy of CAG-LiNGAM is inferior to that of $\mathrm{CAG}^*$ because CAG is less accurate. However, when $p=10$ and the sample size is large, CAG-LiNGAM shows higher precision and $F$-measure than $\mathrm{CAG}^*$. As seen from Figure \ref{ANC} and Table \ref{tab:6}, the accuracy of ancestral relationship estimation improves as the sample size increases. However, even when $(p,|E|,n)=(10,9,500)$, $Anc_{rec}$ is 0.713, which is not close enough to 1. 
From Table \ref{tab:alpha}, when $(p,|E|,n)=(10,9,500)$, $\alpha_P=0.01$ and $\alpha_{CI}=0.001$ are used for the significance levels of the CI tests. Even with a large sample size, if small significance levels are used in the CAG, some ancestral relationships cannot be detected due to type II errors in the CI test. As a result, the number of variables in the estimated groups may be smaller than those in the true CAG grouping, resulting in fewer redundant edges than $\mathrm{CAG}^*$. 
Grouping by CAG is not the finest grouping that guarantees the identifiability of true causal DAGs. The results suggest that using finer groupings may improve the accuracy of causal DAG estimation when the sample size is large.

\section{Conclusion}
\label{sec:5}
This paper proposes CAG as a divide-and-conquer algorithm for learning causal DAGs. This algorithm can help improve the performance of the original DirectLiNGAM, especially when the sample size is small relative to the number of variables. 
CAG is based on the algorithm for finding ancestral relationships among variables in RCD. 
CAG-LiNGAM guarantees the identifiability of true causal DAGs. 
Detailed computer experiments confirm that CAG-LiNGAM may outperform the original DirectLiNGAM, RCD, and CAPA-LiNGAM 
in terms of $F$-measure when the true causal DAG is sparse, and the sample size is small relative to the number of variables. 

If the true causal DAG is connected and the correct CI relationships are estimated, CAG creates the same number of variable groups as sink nodes in the causal DAG. 
Therefore, when the number of sink nodes in the true causal graph is small, 
the number of variables in each group cannot be sufficiently small, 
and thus, improvement in accuracy may not be expected. 
Other computer experiments have shown that even if the true causal DAG is sparse, the estimation accuracy is not necessarily high when the indegrees of some vertices are larger than one. 
CAPA is the opposite of CAG, where the number of variable groups is small when the maximum indegree of the true causal DAG is small, while it may be possible to group variables into many variable groups even when the maximum indegree of the true causal DAG is large. 
As noted at the end of the previous section, the divide-and-conquer algorithm is expected to be more accurate in estimating causal DAGs if finer groupings are used.
The hybrid algorithm of CAG and CAPA may give fine groupings even when the indegree of the true causal DAG is larger. 

We also found that the RCD has higher estimation accuracy than CAG when the sample size is large. RCD estimates causal DAGs using the estimated ancestral relationships among variables. As the sample size increases, the estimation accuracy of ancestral relationships improves. Therefore, using the estimated ancestral relationships to estimate causal DAGs may improve estimation accuracy. However, RCD is computationally expensive when the sample size is large. Using the information on the estimated ancestral relationships in DirectLiNGAM may improve the estimation accuracy of causal DAGs at a relatively low computational cost, even when the sample size is large.

\section*{Acknowledgment}
This work was supported by JSPS KAKENHI Grant Number JP21K11797.

\bibliographystyle{ieeetr}
\bibliography{main.bib}

\section*{Appendix}

\subsection{Pseudo-code of ancestor finding}
\label{sec:Ancestor Finding Code}
This section presents the pseudocode of the ancestor-finding algorithm. 
Lines 5-6 are the procedure that avoids cyclic ancestral relationships and speeds up the ancestor finding. 
$Pearson(x_i, x_j)$ represents the P-value of the Pearson test for the correlation between $x_i$ and $x_j$. 
$KCI(u,v)$ represents the P-value of KCI for the independence between $u$ and $v$. 

We can declare that all ancestral relationships have been identified if no new ancestral relationships are discovered in a single loop because of the no confounder existing assumption.
This means, unlike the original RCD \cite{Maeda2020, Maeda2022}, even if some pairs of variables halt at a point where additional searching for unknown common ancestors is required in Line 35, they will ultimately be treated as having no ancestral relationship.

\begin{breakablealgorithm}
	\renewcommand{\algorithmicrequire}{\textbf{Input:}}
	\renewcommand{\algorithmicensure}{\textbf{Output:}} 
	\caption{Ancestor Finding}
	\label{alg1}
	\begin{algorithmic}[1]
        \REQUIRE The observed data of the variable set $X$
		\ENSURE The list $ANC_{L}$ of the set of ancestors of all variables
        \STATE {Set the significance level of Pearson tests as $\alpha_P$ 
        and the significance level of KCI as $\alpha_{CI}$}
        \STATE {$Anc_i \gets \phi$ and $i = 1, \ldots, p$ \\
        $ANC_L \gets \lbrace Anc_1,\ldots,Anc_p\rbrace$}
        \WHILE{There is no new change in $ANC_L$}
		  \FOR{each pair of variables $x_i$ and $x_j$}
                \IF{$\exists x_k \in Anc_i, x_j \in Anc_k$}
                    \STATE {Add $x_j$ into $Anc_i$}
                \ELSE
                    \IF{$Anc_i \cap Anc_j \neq \phi$}
                        \STATE {$CA \gets Anc_i \cap Anc_j$}
                        \STATE {$v_i \gets x_i - \hat{\bm{\alpha}}_{ij}^\top CA,\quad v_j \gets x_j - \hat{\bm{\alpha}}_{ji}^\top CA$}
                        \STATE{$pvalue_{l} \gets Pearson(v_i, v_j)$ }
                        \IF {$pvalue_{l} \leq \alpha_P$}
                            \STATE {$u_i \gets v_i - \hat{b}_{ij} v_j,\quad u_j \gets v_j - \hat{b}_{ji} v_i$}
                            \STATE{$pvalue^i_j \gets KCI(u_j, v_i)$\\
                            $pvalue^j_i \gets KCI(u_i, v_j)$}
                            \IF{$pvalue^j_i > \alpha_{CI}$ and $pvalue^i_j \leq \alpha_{CI}$}
                                \STATE{Add $x_i$ into $Anc_j$}
                            \ELSIF{$pvalue^j_i \leq \alpha_{CI}$ and $pvalue^i_j > \alpha_{CI}$}
                                 \STATE{Add $x_j$ into $Anc_i$}   
                            \ELSE
                                \STATE {There are unobserved common ancestors between $x_i$ and $x_j$}
                            \ENDIF
                        \ELSE
                            \STATE{No ancestral relationship between $x_i$ and $x_j$}
                        \ENDIF 
                    \ELSE 
                        \STATE{$pvalue_{l} \gets Pearson(x_i, x_j)$}
                        \IF {$pvalue_{l} \leq \alpha_P$}
                            \STATE {$u_i \gets x_i - \hat{b}_{ij} x_j,\quad u_j \gets x_j - \hat{b}_{ji} x_i$}
                            \STATE{$pvalue^i_j \gets KCI(u_j, x_i)$\\
                            $pvalue^j_i \gets KCI(u_i, x_j)$}
                            \IF{$pvalue^j_i > \alpha_{CI}$ and $pvalue^i_j \leq \alpha_{CI}$}
                                \STATE{Add $x_i$ into $Anc_j$}
                            \ELSIF{$pvalue^j_i \le \alpha_{CI}$ and $pvalue^i_j > \alpha_{CI}$}
                                \STATE {Add $x_j$ into $Anc_i$}
                            \ELSE
                                \STATE {There are unobserved common ancestors between $x_i$ and $x_j$}
                            \ENDIF
                        \ELSE
                            \STATE{No ancestral relationship between $x_i$ and $x_j$}
                        \ENDIF
                    \ENDIF
                \ENDIF
		  \ENDFOR
		\ENDWHILE
		\RETURN {$ANC_L$}
	\end{algorithmic}  
\end{breakablealgorithm}

\subsection{Pseudo-code of Grouping and Merging}
This section presents the pseudocode for grouping variables and merging the results. 
In the implementation, a type II error of CI tests can result in many groups with only one variable.
In that case, the output DAG will be too sparse. 
To avoid this, the procedure in lines 3 through 9 merges groups with one element into the other groups. When the sample size is small, no ancestral relationship may be detected, and the number of elements in all groups is one. In such a case, all ${}_pC_2$ variable pairs are defined as the variable grouping.

After grouping, line 13 shows the process of learning causal sub-DAGs and merging them.
Finally, lines 15-17 eliminate cycles in the merged DAG. 

\begin{breakablealgorithm} 
	\renewcommand{\algorithmicrequire}{\textbf{Input:}}
	\renewcommand{\algorithmicensure}{\textbf{Output:}}
    \caption{Grouping and Merging}
	\label{alg2}
	\begin{algorithmic}[1]
	    \REQUIRE  The set $ANC_{L}$ of ancestor lists
		\ENSURE The set of edges $E$
	    \STATE {${\cal V} \gets \emptyset$}
        \STATE Let ${\cal V}_0$ be defined as the family of maximal elements in $ANC_L$ 

        \IF{$|\bm{v}|>1$ for all $\bm{v} \in {\cal V}_0$}
            \STATE{${\cal V} \leftarrow {\cal V}_0$}
        \ELSE
        \FOR {each $\bm{v} \in {\cal V}_0$ with $|\bm{v}|=1$}
                \STATE{Add 
                        $\bigcup_{\bm{v}^\prime \in {\cal V}_0, \bm{v}^\prime \ne \bm{v}} 
                        \{\bm{v} \cup \bm{v}^\prime\}$
                        to
                    ${\cal V}$}
        \ENDFOR
        \ENDIF
        
	\FOR {each $\bm{v}$ in ${\cal V}$}
		\STATE {Estimate the causal sub-DAG $G_{\bm{v}}=(\bm{v},E_{\bm{v}})$ for $\bm{v}$ by DirectLiNGAM}
        \ENDFOR     
        \STATE $E \gets \bigcup_{\bm{v} \in {\cal V}} E_{\bm{v}}$ 
        \STATE $G=(X,E)$
        \IF {$G$ contains cycles}
            \STATE {Eliminate cycles in $G$ }
        \ENDIF
	\RETURN $G$
	\end{algorithmic}  
\end{breakablealgorithm}

\end{document}